\documentclass{article}

\PassOptionsToPackage{dvipsnames}{xcolor}
\PassOptionsToPackage{round,authoryear}{natbib}
\usepackage[preprint]{neurips_2026}

\usepackage[utf8]{inputenc} 
\usepackage[T1]{fontenc}    
\usepackage{hyperref}       
\usepackage{url}            
\usepackage{booktabs}       
\usepackage{amsfonts}       
\usepackage{nicefrac}       
\usepackage{microtype}      
\usepackage{xcolor}         
\usepackage{todonotes}
\usepackage{comment}
\usepackage{float}
\newcommand{\ACD}[1]{\textcolor{purple}{(Aritra: #1)}}
\usepackage{amsmath,amssymb,amsthm,mathtools,bm}
\usepackage{array,booktabs,longtable,tabularx,multirow}
\usepackage{enumitem}
\usepackage{needspace}
\usepackage{siunitx}

\usepackage[dvipsnames]{xcolor}
\usepackage{tikz}
\usetikzlibrary{arrows.meta,positioning,fit,calc,backgrounds}
\usepackage[most]{tcolorbox}
\usepackage{hyperref}
\usepackage[nameinlink,noabbrev]{cleveref}
\usepackage{fancyhdr}

\definecolor{Navy}{HTML}{17324D}
\definecolor{Teal}{HTML}{1F6B70}
\definecolor{SoftBlue}{HTML}{EAF1F7}
\definecolor{SoftTeal}{HTML}{EAF6F4}
\definecolor{SoftGold}{HTML}{FBF5E8}
\definecolor{SoftRed}{HTML}{FBEDEE}
\definecolor{MidGray}{HTML}{5B6570}
\definecolor{LightRule}{HTML}{D6DDE4}

\hypersetup{
  hidelinks,
  pdftitle={Structural Equivariance and Gauge Redundancy in Graph Attention},
  pdfsubject={Local learning coefficient estimation in graph attention models}
}

\newtheorem{theorem}{Theorem}
\newtheorem{proposition}{Proposition}
\newtheorem{lemma}{Lemma}
\newtheorem{corollary}{Corollary}
\theoremstyle{definition}
\newtheorem{definition}{Definition}
\newtheorem{assumption}{Assumption}
\theoremstyle{remark}

\newtcolorbox{keybox}[1][]{
  enhanced,breakable,
  colback=SoftBlue,colframe=Navy,
  boxrule=0.7pt,arc=1.5mm,
  left=2.5mm,right=2.5mm,top=2mm,bottom=2mm,
  title={#1},fonttitle=\bfseries
}
\newtcolorbox{methodbox}[1][]{
  enhanced,breakable,
  colback=SoftTeal,colframe=Teal,
  boxrule=0.7pt,arc=1.5mm,
  left=2.5mm,right=2.5mm,top=2mm,bottom=2mm,
  title={#1},fonttitle=\bfseries
}
\newtcolorbox{cautionbox}[1][]{
  enhanced,breakable,
  colback=SoftGold,colframe=BurntOrange,
  boxrule=0.7pt,arc=1.5mm,
  left=2.5mm,right=2.5mm,top=2mm,bottom=2mm,
  title={#1},fonttitle=\bfseries
}
\newtcolorbox{nonclaimbox}[1][]{
  enhanced,breakable,
  colback=SoftRed,colframe=BrickRed,
  boxrule=0.7pt,arc=1.5mm,
  left=2.5mm,right=2.5mm,top=2mm,bottom=2mm,
  title={#1},fonttitle=\bfseries
}

\newcommand{\R}{\mathbb{R}}
\newcommand{\E}{\mathbb{E}}
\newcommand{\Var}{\operatorname{Var}}

\newcommand{\KL}{\operatorname{KL}}
\newcommand{\rank}{\operatorname{rank}}
\newcommand{\diag}{\operatorname{diag}}
\newcommand{\Aut}{\operatorname{Aut}}
\newcommand{\GL}{\operatorname{GL}}

\newcommand{\ind}{\mathbf{1}}
\newcommand{\cstar}{C_{\star}}
\newcommand{\Astar}{A_{\star}}
\newcommand{\Bstar}{B_{\star}}
\newcommand{\thetastar}{\theta_{\star}}
\newcommand{\Frob}{\mathrm{F}}
\newcommand{\T}{\mathsf{T}}
\newcommand{\dd}{\,\mathrm{d}}
\newcommand{\asympLR}{\mathrel{\asymp}}

\title{\textbf{Symmetries and Singularities}\\[0.45em]}
\workshoptitle{WORKSHOP TITLE}
\author{
  Vishnu Varadarajan \quad Mihir More \quad Aritra Das \quad Debayan Gupta \\
  Truth Audit Labs
}\date{}

\begin{document}

\maketitle
\vspace{-1.8em}
\begin{abstract}
Deep neural networks are highly over-parameterized, and different parameter values represent the same predictive function. This makes their effective complexity difficult to measure using only the number of parameters or the rank of the Hessian. Singular Learning Theory addresses this issue through the local learning coefficient (LLC), which characterizes the effective complexity of a model near a given solution. Existing methods for estimating the LLC often rely on posterior sampling, which can be computationally expensive for large neural networks. This makes accurate LLC estimation difficult at scale. In this work, we use known structures in the model to simplify the analysis and make LLC estimation more tractable. Specifically, we study the LLC of a graph attention model by exploiting symmetries in both the graph structure and the attention parameters. An analytic framework through a teacher--student setting, and explicit LLC estimates after considering the symmetry--induced degeneracies are developed. 

\end{abstract}

\section{Introduction}
With the advent of Deep Learning, the limitations of statistical learning theory in explaining empirically observed behavior such as generalization despite over-parametrization, phase transitions, etc, have been widely noted. Singular Learning Theory (SLT) was proposed as a framework to study the complexities of singular statistical models, and thus deep networks~\citep{Wei_2023}.
SLT identifies the Local Learning Coefficient (LLC) as the key metric of the local effective number of parameters, and the correlation of LLC and these phenomena has been an emerging area of study. \citep{ElliottEtAl2026,CullenEtAl2026} The LLC is the exponent governing the local volume of parameters whose predictive distribution is close to that of a reference parameter~\citep{LauEtAl2025}.  Neural networks are usually singular: several parameter values may implement the same function, the Fisher Information Matrix(FIM) may become degenerate 
and the order of the loss may be more than quadratic in some directions. These features result in the effective number of parameters being a more complicated object than the direct parameter count or the Hessian rank~\citep{Watanabe2009,Watanabe2013}.
An empirical estimator for the LLC based on a localized, tempered posterior has been given by \citet{LauEtAl2025,FurmanLau2024}. Its central computational burden is posterior sampling in a geometry that contains higher-order singular directions. At scale this estimation is computationally hard, and its accuracy is also a concern. The goal of this paper is thus to make this computational burden lesser, and provide better estimations at scale, by exploiting the symmeteries in the model. For this, we look at Graph transformers and look for parameter and input symmetry within this regime. 
We study symmetries at two different levels, (1) of the input graph and (2) of the parameters in the attention block. The graph structure/automorphisms is used to decompose the score-space calculations. The $QK$ block inside the attention module comes with a $GL(k)$ size continuous set of degeneracies, which contributes to the parameter degeneracy. We devise a method that accounts for this degeneracy prior to the LLC estimation. We call this the \emph{$QK$ gauge symmetry}. 

\paragraph{Setup}
For the input graph, we choose one with a known automorphism group, so Cayley graphs become a natural choice. We also assume access to their symmetry group. An arbitrary graph may have a trivial automorphism group and is therefore unsuitable for the decomposition used later.

For measuring the \emph{$QK$ gauge symmetry}, we propose the setup as follows, inspired from the \textbf{wrLLC} setup~\citep{wang2024differentiationspecializationattentionheads}. We consider an ordinary graph-transformer with frozen encoder and the only trainable features as the $QK$ block, that is the $W_Q$ and $W_k$ matrices. To better quantify and track the learning process and enforce realisability, we assume that this model is learning from a teacher model.

\paragraph{Our Results}
In this paper we have motive of developing a parameter/structure symmetry aware analysis of the developmental learning process. We propose a representation theoretic decomposition of Fisher information coming from the graph symmetry in Equation \ref{eq:sector-contributions}. We also identify a key parameter degeneracy in the $QK$ block within the attention mechanism and give explicit reductions in the LLC (Theorem \ref{thm:full-rank}). The teacher-student regime is seen to be the correct framework for analysis, enforcing the realisability and computational criterion. Proposition  \ref{prop:normal-form} and Theorem \ref{thm:rank-r} give a teacher--rank conditioned result, by framing the problem as a reduced--rank martix factorisation problem.

\subsection{Notation}

\begin{table}[H]
\centering
\label{tab:notation}
\begin{tabularx}{\textwidth}{@{}>{\raggedright\arraybackslash}p{0.19\textwidth}X@{}}
\toprule
Symbol & Meaning \\
\midrule
\(G=(V,E)\) & A finite graph \\
\(\Gamma\leq\Aut(G)\) & A chosen graph automorphism group, elements represented by $g$ \\
\(P_g\) & Permutation matrix action of \(g\in\Gamma\) on node-indexed arrays. \\
\(H=E_\eta(G,X)\) & Fixed encoder output in \(\R^{|V|\times d}\), with row \(h_i^{\T}\). \\
\(A,B\) & Notation for \(W_Q,W_K\in\R^{d\times k}\). \\
\(C=AB^{\T}\) & Query--key composite in \(\R^{d\times d}\). \\
\(r\) & Rank of the teacher composite \(\cstar\). \\
\(a=d-r\), \(s=k-r\) & Output and latent dimensions remaining after the rank-\(r\) block. \\
\(F_{\mathrm{score}}\) & Fisher operator in score coordinates. \\
\(J_C,J_\theta\) & Jacobians from composite or factor perturbations to score perturbations. \\
\(\lambda,m\) & Local learning coefficient and multiplicity. \\
\bottomrule
\end{tabularx}
\end{table}

\section{SLT context and notation}

\subsection{Loss and LLC estimation}

Let \(p(y\mid x,\theta)\) be a statistical model, let \(q(y\mid x)\) be the data-generating distribution, and define the population negative log likelihood as
\begin{equation}
  K(\theta)
  =\E_{q(x,y)}\!\left[
      \log\frac{q(y\mid x)}{p(y\mid x,\theta)}
    \right]
  \label{eq:population-K}
\end{equation}
In the realizable teacher--student setting used, \(K\) is a population KL divergence from the teacher and its infimum is zero.

The LLC has multiple formulations, defined initially as the exponent coming out of the normal corssing form of the loss $K$ obtained through blowups and interpreted as the order of the volume scaling around the reference point. For empirical estimations, the Free energy/Marginal likelihood asymptotic expansion formulas are used, and the LLC is interpreted as follows~\citep{Watanabe2009,LauEtAl2025}. 

\begin{definition}[Local LLC/RLCT and multiplicity]
Let \(U\) be a sufficiently small neighbourhood of a minimum \(\thetastar\), and let \(\varphi\) be a smooth density that is strictly positive at \(\thetastar\). Define
\begin{equation}
  Z_U(\tau)
  =\int_U \exp[-\tau K(\theta)]\,\varphi(\theta)\dd\theta.
  \label{eq:local-partition}
\end{equation}
If
\begin{equation}
  Z_U(\tau)
  =C\,\tau^{-\lambda}(\log\tau)^{m-1}\bigl(1+o(1)\bigr),
  \qquad \tau\to\infty,
  \label{eq:partition-asymptotic}
\end{equation}
then \(\lambda\) is the local learning coefficient, equivalently the RLCT under this convention, and \(m\) is its multiplicity.  For estimation, this formulation becomes the most useful. The LLC is a local measure, and so the analysis done below is also localised around the model parameters.
\end{definition}

\subsection{Input Graph}
The Input graph has the aforementioned structure. Let
\begin{equation}
  H=E_\eta(G,X)\in\R^{|V|\times d}
  \label{eq:frozen-encoder}
\end{equation}
be the output of a graph-transformer encoder with frozen parameters \(\eta\). We only require \(E_\eta\) to be equivariant under the action of the graph symmetry, that is,
\begin{equation}
  E_\eta(G,P_gX)=P_gE_\eta(G,X),
  \qquad g\in\Gamma.
  \label{eq:encoder-equivariance}
\end{equation}
Encoders with this property can be constructed in several ways, for instance through weight sharing~\citep{PearceCrumpKnottenbelt2024} or structural encodings~\citep{YingEtAl2021,RampasekEtAl2022}.
Additionally, below we only observe that the joint law of \((G,X,i)\) is \(\Gamma\)-invariant.

\subsection{Attention mechanism and the learning setup}

The graph attention mechanism only slightly deviates from the regular $QK$ attention mechanism. For each node \(i\), the model predicts which node \(i\) attends to, giving a categorical likelihood over \(\mathcal A_i(G)\subseteq V\), the set of nodes visible to \(i\) (for example those within graph distance \(R\)). Let \(M_{ij}\) be the corresponding fixed mask, which is $0$ when $j$ is visible to $i$ and $-\infty$ when not.

The trainable part of the module are the $W_Q$ and $W_k$ matrices. For ease of notation denote,
\begin{equation}
  \qquad A=W_Q,\quad B=W_K,
  \qquad A,B\in\R^{d\times k}, \text{and } C=AB^{\T},
  \label{eq:composite}
\end{equation} where, $d$ represents the embedding dimension and $k$ represents the dimension of the attention head. The logits and attention weights are given as
\begin{align}
  s_{ij}(A,B;H)
   &=\frac{1}{\sqrt{k}}h_i^{\T}AB^{\T}h_j
     +b_{\omega(i,j)}+M_{ij},
  \label{eq:logit}\\
  p_{A,B}(j\mid i,G,X)
   &=\frac{\exp s_{ij}}{\sum_{\ell\in\mathcal A_i(G)}\exp s_{i\ell}},
  \qquad j\in\mathcal A_i(G).
  \label{eq:conditional-softmax}
\end{align}

Our setup includes learning in a teacher--student setting, where the teacher's parameters are given by 
\begin{equation}
  \cstar=\Astar\Bstar^{\T},
  \qquad \rank(\cstar)=r\leq k.
  \label{eq:teacher-composite}
\end{equation}
We draw \((G,X,i)\) from a \(\Gamma\)-invariant distribution. 
Labels are generated from the teacher's attention distribution for the graph nodes \eqref{eq:conditional-softmax}. The model is therefore realizable at \(\cstar\). In addition to realizability, this gives a target to compare an estimated LLC against. The loss is real analytic in \(C\), with \(\cstar\) as an exact minimiser.

Let \(p_\star(\cdot\mid i,G,X)\) denote the teacher distribution generated by \(\cstar\), and define
x\begin{equation}
  K(A,B)
  =\E_{G,X,i}\!
   \left[
     \KL\!\left(
       p_\star(\cdot\mid i,G,X)
       \;\middle\|\;
       p_{A,B}(\cdot\mid i,G,X)
     \right)
   \right].
  \label{eq:attention-KL}
\end{equation}
Because the trainable logits depend on \((A,B)\) only through \(C=AB^{\T}\), write this also as \(\widetilde K(C)\).

For a fixed row distribution \(p\), the softmax Fisher matrix is
\begin{equation}
  \mathcal I(p)=\diag(p)-pp^{\T}.
  \label{eq:softmax-fisher} 
\end{equation}
See~\citep[eq.~(4.106)]{bishop2006} for the derivation.

It is positive semidefinite with kernel spanned by the all-ones vector, reflecting the invariance of softmax to a common row shift.

\section{Structural graph symmetry and score-space decomposition}
{\color{teal}}
\begin{lemma}[Logits are equivariant]
\label{lem:logit-equivariance}
Let us write the logit matrix as $S_{A,B}(H)=\tfrac{1}{\sqrt k}\,HAB^{\T}H^{\T}$.
Then for every $g\in\Gamma$,
\begin{equation}
  S_{A,B}(P_gH)=P_g\,S_{A,B}(H)\,P_g^{\T}.
  \label{eq:logit-equivariance}
\end{equation}
That is, permuting the nodes before computing the logits gives the same
result as computing the logits first and then permuting rows and columns.
\end{lemma}

\begin{proof}
Since $A$ and $B$ are the same for every node,
\[
  S_{A,B}(P_gH)
  =\tfrac{1}{\sqrt k}(P_gH)AB^{\T}(P_gH)^{\T}
  =P_g\bigl(\tfrac{1}{\sqrt k}HAB^{\T}H^{\T}\bigr)P_g^{\T}
  =P_gS_{A,B}(H)P_g^{\T}.
\]
Together with \eqref{eq:encoder-equivariance}, this shows that relabelling
the input nodes by $g$ relabels the logits by $g$ as well.
\end{proof}

The right-hand side of \eqref{eq:logit-equivariance} defines how $\Gamma$
acts on logit matrices,
\begin{equation}
  \rho_{\mathrm{score}}(g)S=P_gSP_g^{\T},
  \label{eq:score-action}
\end{equation}
and this is a representation of $\Gamma$ because $P_{gh}=P_gP_h$.

A finite group maps a parameter to finitely many others, as opposed to a continuous set of parameters with the same predictions. It therefore creates no flat directions. The local partition integral \eqref{eq:local-partition} is multiplied by a constant, and \(\lambda\) is unchanged. We now use the representation to decompose the space of score perturbations into its isotypic components under \(\Gamma\), the subspaces associated with the irreducible representations of \(\Gamma\) on which the score-space Fisher operator acts. 

The Fisher information induced by the probability distribution at the output of the attention block can be expressed as the pullback of the Fisher operator in score space. More precisely, the map from the model parameters $\theta$ to the attention probabilities factors through the logits and the softmax.

In our setup, the loss $K$ in \eqref{eq:attention-KL} is a finite sum over a set
$\mathcal D$ of inputs $(G,X,i)$, and $\mathcal D$ is closed under $\Gamma$. Let
\[
  \mathcal S=\bigoplus_{(G,X,i)\in\mathcal D}\R^{|V|\times|V|}
\]
be the space of logit matrices, one per input. Then $\Gamma$ acts linearly on
$\mathcal S$: $g$ permutes the inputs and conjugates each logit matrix as in
\eqref{eq:score-action}. We call this representation $\rho_{\mathrm{score}}$. The space $\mathcal{S}$ admits a decomposition into isotypic components based on the irreps corresponding to $\rho_{\text{score}}$, $\mathcal{S}=\bigoplus_\rho\mathcal S_\rho$, where $\rho$ runs over the irreducible representations of $\Gamma$~\citep{Serre1977}.

Our architecture is $\Gamma$-equivariant by construction \eqref{eq:encoder-equivariance}. Let
$F_{\mathrm{score}}$ be the Hessian of $K$, viewed as a function of the logits, evaluated against the logits generated by the teacher composite $\cstar$. It acts input by input: on
input $(G,X,i)$ it applies \eqref{eq:softmax-fisher} with $p=p_\star(\cdot\mid i,G,X)$ to
row $i$ of the logit matrix, and vanishes on the remaining rows. Because of the equivariant structure of the graph and its features, $F_{\mathrm{score}}$ commutes with the group action:
\begin{equation}
  F_{\mathrm{score}}\,\rho_{\mathrm{score}}(g)
  =\rho_{\mathrm{score}}(g)\,F_{\mathrm{score}},
  \qquad g\in\Gamma.
  \label{eq:fisher-commutes}
\end{equation}
And hence, $F_{\mathrm{score}}$ acts invariantly on each isotypic component
$\mathcal S_\rho$. This decompostion also gives a natural block diagonal decomposition of the Fisher information itself, through projectors onto these isotypic components. These projectors are given by
\begin{equation}
  \Pi_\rho=\frac{d_\rho}{|\Gamma|}\sum_{g\in\Gamma}
  \overline{\chi_\rho(g)}\,\rho_{\mathrm{score}}(g),
  \label{eq:isotypic-projector}
\end{equation}
where $\chi_\rho$ and $d_\rho$ are the character and dimension of $\rho$. 

Now through composition of $F_{score}$ and $\Pi_\rho$ we get the following quadratic form local representation of the Fisher information.

\subsection{Population KL as a function of the composite}

{
For a composite perturbation \(\Delta C\), the score perturbation is
\begin{equation}
  \Delta s_{ij}=\frac{1}{\sqrt{k}}h_i^{\T}\Delta C h_j.
  \label{eq:score-perturbation}
\end{equation}
Therefore the Hessian quadratic form of \(\widetilde K\), the population KL, near \(\cstar\) is
\begin{equation}
  \mathcal Q(\Delta C)
  =\frac{1}{k}\E_{G,X,i}
   \Var_{J\sim p_\star(\cdot\mid i,G,X)}
   \left[h_i^{\T}\Delta C h_J\right],
  \label{eq:composite-quadratic-form}
\end{equation}
where $\mathcal Q (\Delta C)$ is the Hessian of the KL, that is the FIM, at $\cstar + \Delta C$.
\begin{assumption}[Composite score identifiability]
\label{ass:identifiability}
There is a constant \(\alpha>0\) such that
\begin{equation}
  \mathcal Q(\Delta C)\geq \alpha\|\Delta C\|_{\Frob}^{2}
  \qquad\text{for all }\Delta C\in\R^{d\times d}.
  \label{eq:identifiability}
\end{equation}
The teacher probabilities remain in the interior of the relevant simplices, and the feature moments needed for a third-order Taylor bound are finite in a neighbourhood of \(\cstar\). In addition, \(\widetilde K\) is real analytic near \(\cstar\). This analyticity is automatic for a finite orbit-closed design with teacher probabilities used as soft labels; for a population expectation it follows from an appropriate dominated-analyticity condition.
\end{assumption}

Condition \eqref{eq:identifiability} says that no nonzero composite perturbation is converted, almost surely, only into unobservable constant shifts of each attention row. It is a substantive property of the graph-state distribution and encoder; it must be verified rather than inferred from architecture alone.

Now let $J_C\colon\Delta C\mapsto\Delta S$ be the differential of the map from the
composite to the logits, so that the quadratic form \eqref{eq:composite-quadratic-form}
is $\mathcal Q(\Delta C)=\langle J_C\Delta C,\,F_{\mathrm{score}}\,J_C\Delta C\rangle$.
Since $F_{\mathrm{score}}=\sum_\rho\Pi_\rho F_{\mathrm{score}}\Pi_\rho$,
\begin{equation}
  \mathcal Q=\sum_\rho\mathcal Q_\rho,
  \qquad
  \mathcal Q_\rho(\Delta C)
  =\langle \Pi_\rho J_C\Delta C,\;
    F_{\mathrm{score}}\,\Pi_\rho J_C\Delta C\rangle,
  \label{eq:sector-contributions}
\end{equation}
and each $\mathcal Q_\rho$ is positive semidefinite. Since a sum of nonnegative terms
vanishes only when every term does,
\begin{equation}
  \mathcal Q(\Delta C)=0
  \iff
  \mathcal Q_\rho(\Delta C)=0\ \text{ for every }\rho.
  \label{eq:kernel-intersection}
\end{equation}
Thus \cref{ass:identifiability} fails exactly when some nonzero $\Delta C$ satisfies
$\mathcal Q_\rho(\Delta C)=0$ for every $\rho$. The decomposition then shows where and
how: in each component, either $\Pi_\rho J_C\Delta C=0$, so $\Delta C$ does not change
the logits in that component at all, or $F_{\mathrm{score}}\,\Pi_\rho J_C\Delta C=0$, so
it adds the same constant to every logit in a row, which leaves the softmax unchanged.
}

\section{Guage symmetry and rank-deficient teacher}

For every \(\xi\in\GL(k)\), consider the map 
\begin{equation}
  (A,B)\longmapsto (A\xi,B\xi^{-\T})
  \label{eq:gauge-action}
\end{equation}
Then
\begin{equation}
  (A\xi)(B\xi^{-\T})^{\T}=AB^{\T}=C,
  \label{eq:gauge-invariance-composite}
\end{equation}
so every logit and predictive probability is unchanged for every input. This is a parameter-space redundancy, or \emph{gauge symmetry}. 
The same algebra occurs in the value--output composite,
\[
  W_V\mapsto W_V\xi,
  \qquad W_O\mapsto \xi^{-1}W_O,
\]
however we restrict ourselves to the $QK$ attention block. 

For a curve \(G(t)=I_k+tX+O(t^2)\), the infinitesimal gauge direction is
\begin{equation}
  \delta A=AX,
  \qquad
  \delta B=-BX^{\T},
  \qquad X\in\R^{k\times k}.
  \label{eq:gauge-tangent}
\end{equation}
The derivative of the composite map \(\pi(A,B)=AB^{\T}\) is
\begin{equation}
  D\pi_{A,B}(\Delta A,\Delta B)
  =\Delta A B^{\T}+A\Delta B^{\T}.
  \label{eq:composite-derivative}
\end{equation}
Substitution of \eqref{eq:gauge-tangent} gives zero. Consequently, gauge tangents lie in the kernel of the Jacobian from parameters to logits and in the kernel of the parameter Fisher information.

\begin{lemma}[Local equivalence of the attention KL and composite error]
\label{prop:local-equivalence}
Under \cref{ass:identifiability}, there is a neighbourhood \(\mathcal N\) of \(\cstar\) and constants \(0<c_1\leq c_2<\infty\) such that
\begin{equation}
  c_1\|C-\cstar\|_{\Frob}^{2}
  \leq \widetilde K(C)
  \leq c_2\|C-\cstar\|_{\Frob}^{2},
  \qquad C\in\mathcal N.
  \label{eq:KL-equivalence}
\end{equation}
Consequently,
\begin{equation}
  K(A,B)\asympLR \|AB^{\T}-\cstar\|_{\Frob}^{2}
  \label{eq:factor-loss-equivalence}
\end{equation}
near any factor pair representing \(\cstar\). With a smooth positive local density, the two functions have the same local LLC and multiplicity.
\end{lemma}

\begin{proof}
The gradient of \(\widetilde K\) vanishes at \(\cstar\). Its Hessian is the positive-definite quadratic form \(\mathcal Q\) in \eqref{eq:composite-quadratic-form}. Taylor's theorem and the finite third-order remainder yield a lower and upper quadratic bound in a sufficiently small neighbourhood. Composing with the polynomial map \((A,B)\mapsto AB^{\T}\) gives \eqref{eq:factor-loss-equivalence}. Two nonnegative analytic functions bounded above and below by positive constant multiples have the same leading local partition exponent and multiplicity; this is a standard invariance used in singular learning theory~\citep{Watanabe2009}.
\end{proof}

Assumption \ref{ass:identifiability} makes the map from \(C\) to distributions locally nonsingular modulo softmax row shifts. If this condition fails, the formula in the below sections, for example \ref{thm:rank-r} where one sees this explicitly, need not hold true.

The intent of Lemma \ref{prop:local-equivalence} is to reduce the analysis to that of
\begin{equation}
  K_{\mathrm{mat}}(A,B)=\|AB^{\T}-\cstar\|_{\Frob}^{2}.
  \label{eq:matrix-factor-loss}
\end{equation}
Assume throughout this section that \(d\geq k\) and that the local density in the $A, B$ coordinates is smooth and strictly positive at the specified reference point.

\begin{theorem}[Full-rank local coefficient]
\label{thm:full-rank}
Suppose \(\rank(\cstar)=k\) and \(\Astar,\Bstar\in\R^{d\times k}\) both have full column rank. Then the gauge action \eqref{eq:gauge-action} is locally free, and
\begin{equation}
  \lambda_{\mathrm{full}}
  =\frac{2dk-k^2}{2}
  \qquad m_{\mathrm{full}}=1.
  \label{eq:full-rank-llc}
\end{equation}
\end{theorem}

\begin{proof}
The parameter dimension is \(2dk\). At a full-column-rank pair, the kernel of \(D\pi\) in \eqref{eq:composite-derivative} is exactly the \(k^2\)-dimensional gauge tangent space \eqref{eq:gauge-tangent}. The image is the tangent space of the rank-\(k\) matrix manifold, whose dimension is
\[
  q=2dk-k^2.
\]
The constant-rank theorem gives local coordinates \((u,t)\in\R^q\times\R^{k^2}\) in which the composite depends regularly on \(u\) and not on \(t\). Hence \(K_{\mathrm{mat}}\asympLR\|u\|^2\), which has coefficient \(q/2\) and multiplicity one.
\end{proof}

This first case is a test of whether an estimator handles continuous non-identifiability and gauge-dependent conditioning, but it does not yet contain the higher-order rank singularity.

\subsection{Rank deficiency in the teacher}

To handle the rank-deficient case $(r<k)$, we choose orthogonal coordinates aligned with the singular value decomposition of the teacher matrix ($\cstar$). In general, if
\[
\cstar = U_r\Sigma_r V_r^\top
\]
is a rank-(r) matrix, then a natural balanced factorization is obtained by splitting the nonzero singular values symmetrically as
\[
\Astar = U_r\Sigma_r^{1/2},
\qquad
\Bstar = V_r\Sigma_r^{1/2},
\]
so that
\[
\Astar\Bstar^\top=\cstar.
\]
When $(r<k)$, the factorization has $(k-r)$ additional directions corresponding to the zero singular values. Thus, after choosing separate orthogonal row and column coordinates aligned with the singular vectors of ($\cstar$), we may work in canonical coordinates in which the teacher takes the form
Let \(r<k\), set \(a=d-r\) and \(s=k-r\), and use separate orthogonal row and column coordinates for \(\cstar\) to write
\begin{equation}
  \cstar=
  \begin{bmatrix}
    \Sigma_r&0\\0&0
  \end{bmatrix},
  \qquad \Sigma_r=\diag(\sigma_1,\ldots,\sigma_r),
  \quad \sigma_i>0.
  \label{eq:canonical-C}
\end{equation}
The canonical balanced factorization is
\begin{equation}
  \Astar=\Bstar=
  \begin{bmatrix}
    \Sigma_r^{1/2}&0\\0&0
  \end{bmatrix}
  \in\R^{d\times k}.
  \label{eq:canonical-factors}
\end{equation}
The row blocks have sizes \(r,a\), and the latent-column blocks have sizes \(r,s\).

Write perturbations as
\begin{equation}
  \widetilde \Astar =\Astar+
  \begin{bmatrix}A_{11}&A_{12}\\A_{21}&A_{22}\end{bmatrix},
  \qquad
  \widetilde \Bstar=\Bstar+
  \begin{bmatrix}B_{11}&B_{12}\\B_{21}&B_{22}\end{bmatrix}.
  \label{eq:block-perturbations}
\end{equation}
The first-order change in the composite is 
\begin{align*}
    \Delta C &=(\Astar+\Delta A)(\Bstar+\Delta B)^{\T}- \Astar \Bstar^{\T}\\
    &=A^\star\Delta B^{\T}+\Delta AB^{\star\T}+\Delta A\Delta B^{\T}\\
    &=
    \begin{bmatrix}
      A_{11}\Sigma_r^{1/2}+\Sigma_r^{1/2}B_{11}^{\T}&\Sigma_r^{1/2}B_{21}^{\T}\\
      A_{21}\Sigma_r^{1/2}&0
    \end{bmatrix}
    +
    \begin{bmatrix}
      A_{12}B_{12}^{\T}&A_{12}B_{22}^{\T}\\A_{22}B_{12}^{\T}&A_{22}B_{22}^{\T}
    \end{bmatrix}\\
    &\approx
    \begin{bmatrix}
      A_{11}\Sigma_r^{1/2}+\Sigma_r^{1/2}B_{11}^{\T}&\Sigma_r^{1/2}B_{21}^{\T}\\
      A_{21}\Sigma_r^{1/2}&0
    \end{bmatrix} 
\end{align*}
 This is precisely the derivative of the composite map, $D\pi_{A,B}(\Delta A, \Delta B)$. The rank of the derivative $D\pi_{A,B}$ is thus the dimension of the image, which has $\Delta C_{11}, \ \Delta C_{21}, \ \Delta C_{12}$ as the free coordinates, and thus equal to
\begin{equation}
  r^2+ra+ar=r(2d-r),
  \label{eq:regular-rank}
\end{equation}
which is the dimension of the rank-\(r\) matrix manifold. First-order analysis identifies the regular directions and the loss of quadratic curvature, but it does not determine the LLC of the remaining directions. Their higher-order contribution is obtained from the Schur complement.

\subsection{Schur-complement transverse coordinate}

For a composite matrix near \(\cstar\), write

\begin{equation}
  C=\begin{bmatrix}C_{11}&C_{12}\\C_{21}&C_{22}\end{bmatrix}.
\end{equation}
The block \(C_{11}\) remains invertible in a sufficiently small neighbourhood. Define
\begin{equation}
  R=C_{22}-C_{21}C_{11}^{-1}C_{12}.
  \label{eq:schur-complement}
\end{equation}
The map
\begin{equation}
  C\longmapsto (C_{11},C_{12},C_{21},R)
  \label{eq:schur-coordinate-map}
\end{equation}
 is a local analytic diffeomorphism. Moreover,
\begin{equation}
  \rank(C)=r+\rank(R)
  \label{eq:rank-schur}
\end{equation}
near \(\cstar\). Therefore \(R=0\) is the local rank-\(r\) manifold, and \(R\) measures the rank created transversely to that manifold. Splitting the factors by output rows,
\[
  A=\begin{bmatrix}A_1\\A_2\end{bmatrix},
  \qquad
  B=\begin{bmatrix}B_1\\B_2\end{bmatrix},
  \qquad A_1,B_1\in\R^{r\times k}.
\]
Then
\begin{equation}
  R=A_2PB_2^{\T},
  \qquad
  P=I_k-B_1^{\T}(A_1B_1^{\T})^{-1}A_1.
  \label{eq:R-factor-P}
\end{equation}
The matrix \(P\) is idempotent and has constant rank \(s=k-r\) near the canonical point. A local analytic choice of bases for its image and coimage factors it as \(P=UV_0\), with \(U\in\R^{k\times s}\), \(V_0\in\R^{s\times k}\), and \(V_0U=I_s\). Consequently,
\begin{equation}
  R=(A_2U)(V_0B_2^{\T})=VS,
  \qquad
  V\in\R^{a\times s},\quad S\in\R^{s\times a}.
  \label{eq:VS-block}
\end{equation}

\begin{proposition}[Rank-stratified local normal form]
\label{prop:normal-form}
At the canonical factorization \eqref{eq:canonical-factors}, there are local analytic coordinates consisting of
\begin{equation}
  u\in\R^{r(2d-r)},
  \quad V\in\R^{a\times s},
  \quad S\in\R^{s\times a},
  \quad t\in\R^{r(2k-r)},
  \label{eq:normal-coordinates-dim}
\end{equation}
for which
\begin{equation}
  K_{\mathrm{mat}}(u,V,S,t)
  \asympLR \|u\|^2+\|VS\|_{\Frob}^{2}.
  \label{eq:normal-form}
\end{equation}
The coordinate \(t\) is flat and parametrizes the gauge-orbit directions that are nontrivial at the canonical point. The residual \(\GL(s)\) stabilizer acts inside the \((V,S)\) block and is therefore already included in its reduced-rank singularity.
\end{proposition}
The Schur map \eqref{eq:schur-coordinate-map} is a local analytic diffeomorphism, so \(\|C-\cstar\|_{\Frob}^{2}\) is locally equivalent to the sum of the squared regular block coordinates and \(\|R\|_{\Frob}^{2}\). A local gauge section can be constructed explicitly. Split the top row blocks as \(A_1=[A_{1a}\ A_{1b}]\) and \(B_1=[B_{1a}\ B_{1b}]\), with \(A_{1a},B_{1a}\in\R^{r\times r}\). Near the canonical point these active blocks are invertible. Successive gauge transformations set
\[
  A_1=[\Sigma_r^{1/2}\ 0],
  \qquad B_1=[\widehat B_{1a}\ 0],
\]
while retaining the inactive \(\GL(s)\) freedom. On this section,
\begin{align*}
  C_{11}&=\Sigma_r^{1/2}\widehat B_{1a}^{\T},
  &C_{12}&=\Sigma_r^{1/2}B_{2a}^{\T},\\
  C_{21}&=A_{2a}\widehat B_{1a}^{\T},
  &R&=A_{2b}B_{2b}^{\T}.
\end{align*}
Thus \((C_{11}-\Sigma_r,C_{12},C_{21})\) provide \(r(2d-r)\) regular coordinates, and one may take \(V=A_{2b}\), \(S=B_{2b}^{\T}\). The omitted coset coordinates have dimension \(r(2k-r)\) and do not enter the composite because of gauge invariance. The explicit elimination and the stabilizer count are given in \cref{sec:proof-slice,sec:proof-stabilizer}.

The dimension count is exact:
\begin{equation}
  r(2d-r)+2(d-r)(k-r)+r(2k-r)=2dk.
  \label{eq:dimension-identity}
\end{equation}
To see the gauge term, the stabilizer of \((\Astar,\Bstar)\) inside \(\GL(k)\) is isomorphic to \(\GL(s)\) and has dimension \(s^2\). Hence the orbit dimension is
\begin{equation}
  k^2-s^2=r(2k-r).
  \label{eq:gauge-orbit-rank-r}
\end{equation}
This explains why subtracting \(k^2\) from the parameter count is wrong when \(r<k\): the stabilizer is larger, and the remaining \(2as\) coordinates enter through the higher-order product \(VS\).

\subsection{Exact coefficient and multiplicity}

The zero-target block
\begin{equation}
  K_0(V,S)=\|VS\|_{\Frob}^{2},
  \qquad V\in\R^{a\times s},\quad S\in\R^{s\times a},
  \label{eq:zero-target-block}
\end{equation}
 is the square, symmetric-outer-dimension specialization of reduced-rank matrix factorization. The exact reduced-rank RLCT and multiplicity were determined by \citet{AoyagiWatanabe2005}.

\begin{theorem}[Exact canonical rank-\(r\) LLC]
\label{thm:rank-r}
Let \(0\leq r\leq k\leq d\), and let the teacher factors be the canonical balanced pair \eqref{eq:canonical-factors}. Under \cref{ass:identifiability}, the local LLC and multiplicity of the graph-attention model are
\begin{align}
  \lambda_r
  &=\frac{r(2d-r)}{2}
    +\frac{(d-r)(k-r)-\lfloor (k-r)^2/4\rfloor}{2},
  \label{eq:lambda-r-decomposed}\\
  m_r
  &=\begin{cases}
      1,&k-r\text{ is even},\\
      2,&k-r\text{ is odd}.
    \end{cases}
  \label{eq:m-r}
\end{align}
Equivalently,
\begin{equation}
  \lambda_r
  =\frac{d^2}{2}
   -\frac{(2d-k-r)^2}{8}
   +\frac{1}{8}\ind_{\{k-r\text{ odd}\}}.
  \label{eq:lambda-r-compact}
\end{equation}
\end{theorem}

\begin{proof}
By \cref{prop:normal-form}, the regular block contributes \(r(2d-r)/2\), the gauge-flat block contributes zero, and the remaining contribution is the RLCT of \eqref{eq:zero-target-block}. Specializing the Aoyagi--Watanabe formula to outer dimensions \(M=N=a\), hidden width \(H=s\), and true rank zero gives
\begin{equation}
  \lambda_0(a,s)
  =\frac{as-\lfloor s^2/4\rfloor}{2}
  =\frac{s(4a-s)}{8}
    +\frac18\ind_{\{s\text{ odd}\}},
  \label{eq:lambda-zero}
\end{equation}
with multiplicity two precisely when \(s\) is odd. Adding the regular block proves \eqref{eq:lambda-r-decomposed}; elementary algebra gives \eqref{eq:lambda-r-compact}.
\end{proof}

\begin{corollary}[Gauge-orbit invariance]
\label{cor:gauge-orbit-llc}
For every fixed \(\xi\in\GL(k)\), the same pair \((\lambda_r,m_r)\) holds locally at
\[
  (\Astar\xi,\Bstar\xi^{-\T}).
\]
\end{corollary}

\begin{proof}
The gauge map is an analytic diffeomorphism of factor space and leaves \(AB^{\T}\), hence the loss, unchanged. Its Jacobian determinant on the full \((A,B)\) space is \(\det(\xi)^d\det(\xi)^{-d}=1\), so a smooth positive local density remains smooth and positive after the coordinate change. The local exponent and pole multiplicity are therefore unchanged.
\end{proof}

For odd \(s\), the two adjacent integers \((s-1)/2\) and \((s+1)/2\) maximize the rank-stratum term \(j(s-j)\) in the resolution calculation. This tie is reflected in the order-two pole and the additional \(\log\tau\) factor. For even \(s\), there is a unique maximizer and the multiplicity is one.

\section{Limitations and Future Work}

Our result holds at the teacher pair \eqref{eq:canonical-factors}, where the unused columns of \(A\) and \(B\) are zero. If \(r<k\), there are other pairs with the same product \(\cstar\), where the unused columns are nonzero and cancel each other. We do not compute the LLC at those points. Since the exact LLC values are known, it would be useful to run the SGLD estimator introduced by \citet{LauEtAl2025}, as well as MALA~\citep{RobertsRosenthal1998} and HMC~\citep{Neal2011}, both to validate our results empirically and to check how well these estimators work. Finally, extending our analysis to a full transformer would require more work. Further, the connection to the broader Symmetry-Constrained Learning literature is interesting.


\bibliographystyle{plainnat}
\bibliography{references}
\appendix
\section{Proof details for the rank-stratified normal form}

\subsection{First-order regular coordinates}

At the canonical point \eqref{eq:canonical-factors}, the derivative in \eqref{eq:composite-derivative} has the block form 
The maps to \(\Delta C_{12}\) and \(\Delta C_{21}\) are surjective because \(\Sigma_r^{1/2}\) is invertible. The map
\[
  (A_{11},B_{11})
  \longmapsto
  A_{11}\Sigma_r^{1/2}+\Sigma_r^{1/2}B_{11}^{\T}
\]
is surjective onto \(\R^{r\times r}\). Hence the derivative image dimension is \(r^2+2ar=r(2d-r)\). These output coordinates can be completed to regular local coordinates by the implicit-function theorem.

The first-order kernel is larger than the gauge tangent. Its excess dimension is
\begin{align}
  \left[2dk-r(2d-r)\right]-r(2k-r)
  &=2(d-r)(k-r)\\
  &=2as,
\end{align}
which is precisely the number of entries in \((V,S)\). These directions are invisible to the Hessian but are not flat to all orders.

\subsection{Schur-complement identity}

Block Gaussian elimination gives
\begin{equation}
  \begin{bmatrix}
    I&0\\-C_{21}C_{11}^{-1}&I
  \end{bmatrix}
  \begin{bmatrix}
    C_{11}&C_{12}\\C_{21}&C_{22}
  \end{bmatrix}
  =
  \begin{bmatrix}
    C_{11}&C_{12}\\0&R
  \end{bmatrix}.
  \label{eq:block-elimination}
\end{equation}
The left factor is invertible, so \(\rank(C)=\rank(C_{11})+\rank(R)=r+\rank(R)\) near the teacher. The inverse of \eqref{eq:schur-coordinate-map} is
\begin{equation}
  C_{22}=R+C_{21}C_{11}^{-1}C_{12},
\end{equation}
which is analytic while \(C_{11}\) is invertible. Therefore the coordinate change preserves the local LLC.

Substituting \(C_{ij}=A_iB_j^{\T}\) gives
\begin{align}
  R
  &=A_2B_2^{\T}
    -A_2B_1^{\T}(A_1B_1^{\T})^{-1}A_1B_2^{\T}\\
  &=A_2\left[I_k-B_1^{\T}(A_1B_1^{\T})^{-1}A_1\right]B_2^{\T}.
\end{align}
For
\[
  Q=B_1^{\T}(A_1B_1^{\T})^{-1}A_1,
\]
one has \(Q^2=Q\) and \(\rank(Q)=r\). Thus \(P=I-Q\) is idempotent of rank \(s\). Constant-rank analytic matrix families admit local analytic frames for their image and kernel, yielding the factorization used in \eqref{eq:VS-block}.

\subsection{Explicit local gauge section}
\label{sec:proof-slice}
Write the top row blocks as
\[
  A_1=[A_{1a}\ A_{1b}],
  \qquad B_1=[B_{1a}\ B_{1b}],
\]
where the active \(r\times r\) blocks are invertible near the canonical point, and set \(D=\Sigma_r^{1/2}\). The gauge matrix
\[
  \xi_1=
  \begin{bmatrix}
    A_{1a}^{-1}D&-A_{1a}^{-1}A_{1b}\\
    0&I_s
  \end{bmatrix}
\]
sets \(A_1\xi_1=[D\ 0]\). After applying this transformation to both factors, write the transformed top block of \(B\) again as \([B_{1a}\ B_{1b}]\). The matrix
\[
  \xi_2=
  \begin{bmatrix}I_r&0\\X&I_s\end{bmatrix},
  \qquad X^{\T}=B_{1a}^{-1}B_{1b},
\]
preserves \([D\ 0]\) and sends \(B_1\xi_2^{-\T}\) to \([B_{1a}\ 0]\). Both transformations depend analytically on the original factors.

On this section, split the lower row blocks as \(A_2=[A_{2a}\ A_{2b}]\) and \(B_2=[B_{2a}\ B_{2b}]\). Then
\begin{align}
  C_{11}&=DB_{1a}^{\T},
  &C_{12}&=DB_{2a}^{\T},\\
  C_{21}&=A_{2a}B_{1a}^{\T},
  &C_{22}&=A_{2a}B_{2a}^{\T}+A_{2b}B_{2b}^{\T}.
\end{align}
Consequently,
\[
  A_{2a}=C_{21}C_{11}^{-1}D,
  \qquad B_{2a}^{\T}=D^{-1}C_{12},
  \qquad R=A_{2b}B_{2b}^{\T}.
\]
Thus the section variables are analytically equivalent to the regular composite blocks together with \(V=A_{2b}\) and \(S=B_{2b}^{\T}\). The parameters used in \(\xi_1\) and the lower-left block of \(\xi_2\) supply \(r^2+2rs=r(2k-r)\) orbit coordinates. The remaining block-diagonal \(\GL(s)\) action is the stabilizer at the canonical point and acts within \((V,S)\).

\subsection{Gauge stabilizer at the canonical point}
\label{sec:proof-stabilizer}

Write a gauge matrix \(\xi\) in active/redundant blocks. The condition \(\Astar\xi=\Astar\) forces the first block row of \(\xi\) to be \([I_r\;0]\). The condition \(\Bstar\xi^{-\T}=\Bstar\) then forces the lower-left block to vanish. Hence
\begin{equation}
  \operatorname{Stab}(\Astar,\Bstar)
  =\left\{
    \begin{bmatrix}I_r&0\\0&G_{22}\end{bmatrix}
    :G_{22}\in\GL(s)
   \right\}.
\end{equation}
Its dimension is \(s^2\), and the orbit dimension is \(k^2-s^2=r(2k-r)\), as stated in \eqref{eq:gauge-orbit-rank-r}.

\subsection{Specialization of the reduced-rank formula}

The Aoyagi--Watanabe reduced-rank theorem for an \(M\times H\) factor times an \(H\times N\) factor, with true product rank \(r_0\), has a central regime defined by
\[
  N+r_0\leq M+H,
  \quad M+r_0\leq N+H,
  \quad H+r_0\leq M+N.
\]
For the zero-target block \eqref{eq:zero-target-block}, set \(M=N=a\), \(H=s\), and \(r_0=0\). Since \(s\leq a\), all three inequalities hold. If \(s\) is even, the theorem gives
\begin{align}
  \lambda_0
  &=\frac{-s^2-a^2-a^2+2\{s(2a)+a^2\}}{8}\\
  &=\frac{4as-s^2}{8}.
\end{align}
If \(s\) is odd, the numerator gains one. These two cases are equivalent to \eqref{eq:lambda-zero}. The multiplicity is two in the odd case and one otherwise~\citep{AoyagiWatanabe2005}.

\IfFileExists{checklist.tex}{%
  \newpage
  \input{checklist.tex}
}{%
  \typeout{WARNING: checklist.tex not found; NeurIPS 2026 requires the official checklist for submission.}%
}

\end{document}